\documentclass{article} %
\usepackage{iclr2027_conference,times}
\usepackage[utf8]{inputenc} %
\usepackage[T1]{fontenc}    %
\usepackage{makecell}
\usepackage{wrapfig}
\usepackage{hyperref}       %
\usepackage{url}            %
\usepackage{graphicx}
\usepackage{subcaption}
\usepackage{booktabs}       %
\usepackage{amsmath, amsthm, amssymb, amsfonts} %
\usepackage{subcaption}
\usepackage{algorithm, algorithmic}
\usepackage{nicefrac}       %
\usepackage{microtype}      %
\usepackage{xcolor}         %
\newtheorem{myprop}{Proposition}
 
\newtheorem{mylem}{Lemma}
\usepackage[normalem]{ulem}
\usepackage{verbatim}

\usepackage{amsmath,amsfonts,bm}

\def\eqref#1{equation~\ref{#1}}

\def\1{\bm{1}}

\def\vu{{\bm{u}}}

\def\vx{{\bm{x}}}
\def\vy{{\bm{y}}}
\def\vz{{\bm{z}}}

\DeclareMathAlphabet{\mathsfit}{\encodingdefault}{\sfdefault}{m}{sl}
\SetMathAlphabet{\mathsfit}{bold}{\encodingdefault}{\sfdefault}{bx}{n}

\title{Mutual Equilibrium: Multimodal Representation Learning through Reciprocal Feedback}

\author{Ho-min Park \\
Data Science Center, Texas Children's Hospital \\
Baylor College of Medicine \\
Houston, Texas 77030, USA \\
\texttt{Ho-min.Park@bcm.edu} \\
\And
Byungkon Kang \\
Department of Computer Science\\
SUNY Korea \\
Incheon, Republic of Korea \\
\texttt{byungkon.kang@sunykorea.ac.kr}
}

\newcommand{\modelname}{MEQ}
\newcommand{\fzx}{F_{\vz_x^*}}
\newcommand{\fzy}{F_{\vz_y^*}}
\newcommand{\gzx}{G_{\vz_x^*}}
\newcommand{\gzy}{G_{\vz_y^*}}

\iclrfinalcopy
\begin{document}

\maketitle

\begin{abstract}
This work proposes a mutual feedback architecture, \modelname{}, that refines the two inputs, of possibly different modalities, into a pair of coupled embeddings such that each embedding reflects the information of the other. The core idea is to incorporate continuous interchange of information between the two inputs. This idea leads to a mutual feedback architecture consisting of two components whose outputs are fed back into the other. The final output of this model is defined as the fixed point of this interaction. We provide theoretical analysis that offers interpretation of this model as well as design choices to prevent failure cases. We show the benefits of \modelname{} through classification and visual grounding tasks spanning various datasets. Quantitatively, our model outperforms or shows competitive performance on concatenation-based multimodal classification problems. Qualitatively, the proposed interactive mechanism allows the model to progressively refine the visual grounding when paired with complementary modality, thus demonstrating the power of mutual feedback under such settings.
\end{abstract}

\section{Introduction}
Consider processing two inputs, $\vx$ and $\vy$, of possibly different modalities. The goal is to derive a representation that accounts for both inputs in the sense that the individual information are combined appropriately so as to be useful for a task involving both modalities. This problem is addressed in many disciplines such as machine learning, modality fusion, and data mining. It has become one of the most important preprocessing steps in multimodal domains.\\
A de facto standard way of approaching this problem is to concatenate $\vx$ and $\vy$ before passing them through a neural network in the hope that the information are sufficiently fused. However, existing methods typically treat the individual features as fixed descriptors and combine them in a single feed-forward operation. Such formulations implicitly assume that each extracted embedding is already an optimal representation. But complementary modalities often provide information that should alter how the other modality is interpreted. For instance, an image of a tiger hidden in bushes paired with a text reading ``A tiger staying hidden'' should have a representation that reflects the `tiger-ness' more strongly than one without such text. Conversely, the corresponding text representation should also be updated to include the information that the tiger is hiding in bushes.

We therefore argue that multimodal representation learning is more naturally viewed as a process where inter-modality influence plays an important role. Specifically, we propose an iterative process in which modality-specific representations repeatedly influence one another until reaching a mutually consistent state. This idea resembles how two humans update their beliefs through repetitive exchange until their beliefs agree. It is this iterative influence that we aim to model mathematically and derive an algorithm that can refine the two inputs into another pair where each new representation will reflect information in the other.\\
Furthermore, we propose to model this iterative procedure as a pair of neural networks that feed each other, where one's output becomes part of the input to the other (Fig.~\ref{fig:diag}). The final output of this architecture is the \emph{fixed point} of this iteration, which is a natural way to embody the notion of agreement: when the information exchange no longer updates each other's belief, an agreement is said to be achieved. Thus, our central contribution becomes a mutually interactive system that achieves equilibrium.\\
Although we primarily investigate our algorithm in a representation learning framework, such an iterative refinement approach can have numerous applications in other areas. For example, cases when we need to correct conflicting sensor inputs, or improve LLMs through iteration are our intended future works.
\subsection{Related Works and Contributions}
\label{relworks}
In this section, we provide a list of works relevant to ours, focusing on the following three areas that are related to our topic.
\paragraph{Multimodal fusion} Although we do not deal with multimodal fusion directly, we review this field as it somewhat overlaps with ours in application domain. Multimodal fusion has long been considered the essential step in multimodal learning. While there are numerous prior work dealing with multimodal fusion itself, a good majority of them attempt variations of concatenation of features~\citep{mm_survey, mm_rl_survey}. However, there are a few that share the spirit of our approach. The tensor fusion network (TFN:~\cite{tfn}) is one of the earlier works that address fusion by means of intra- and inter-modality connection instead of relying on simple concatenation. Co-attention network~\citep{coatt} and LXMERT~\citep{lxmert} follow a similar path, except they specifically use cross-modal attention layers in place of TFN's outer products, to train text and vision encoders via cross-attention. %
\textcolor{black}{Also, TFN's high computational complexity was later addressed by low-rank multimodal fusion (LMF: \citet{lmf}), while multimodal compact bilinear pooling (MCB: \citet{mcb}) had earlier reduced the cost of bilinear interaction by sketching the outer product.}
MulT~\citep{muit} extended the idea to stacks of multimodal cross-attentions in a Transformer architecture to achieve repeated interaction between modalities. However, MulT limits the number of interactions to a fixed finite number, whereas we approach it with a fixed point.
\paragraph{Modality interaction} The idea of inputs interacting with each other has also been used in other settings. For example, self-supervised approaches like BYOL~\citep{byol} and SimSiam~\citep{simsiam} demonstrate the value of  repeated information exchange. However, such approaches also perform finite number of interactions. In addition, such a concept has also been extended to mixture of experts~\citep{interaction_moe} and graph node classification~\citep{influence_mod}.\\
Several more recent works have expanded the fusion domain to that of large language models (LLMs). In particular, it is worth noting that the concept of \emph{multi-agent LLMs} was proposed in the context of repeated information exchange. These works are based on the idea that multiple LLMs can refine the output by interacting with each other's outputs~\citep{moa,madebate}. The crucial difference between these works and ours is that first, we operate in representation space, widening the applicability, and secondly we theoretically analyze the dynamic coupling of the two encoders.
\paragraph{Deep equilibrium models} Heavily based on the concept of a fixed point, our model naturally falls into the deep equilibrium model (DEQ:~\cite{deq}) family. These are models that return the fixed point of iterated hidden state updates as outputs. We will borrow common tools such as proofs and Jacobian analysis~\citep{jacc_eq,mdeq} from such works when analyzing our algorithm. Of particular interest among these models is the work on deep equilibrium for multimodal fusion~\citep{deqmm}. While this work does share a theme with our work, the main approach relies on a simple DEQ application to weighted sum of features (i.e., a single hidden state). Again, our work focuses on how the coupled dynamics arise as the two modalities interact\textcolor{black}{, and we compare against that fusion design directly in Section~\ref{sec:ablation}.}

That said, our contributions can be summarized as follows.
\begin{itemize}
    \item We formulate mutually interactive feature learning as a coupled dynamic system, whose final outcome is the fixed point of the interaction. The interaction is performed in representation space rather than in output space, allowing for further applications.
    \item Theoretical properties of the said system are identified and used in deriving the main algorithm. This also results in a framework on which variant approaches can be based by allowing for different modules that fit the theoretical profile.
    \item Our experiments verify that such an equilibrium is meaningful in a variety of prediction tasks. Moreover, we show there is value in iterative refinement itself.
\end{itemize}
\section{Main Approach} 
\label{mainapp}
Given two inputs (or modalities) $\vx\in\mathcal{X}\subseteq\mathbb{R}^n$ and $\vy\in\mathcal{Y}\subseteq\mathbb{R}^m$, the goal is to generate a pair of embeddings $\vz_x\in\mathbb{R}^k$ and $\vz_y\in\mathbb{R}^k$ such that they are mutual reflections of the other inputs\footnote{We discuss the case of more than two modalities in Appendix~\ref{app:multmod}}.\\
The high-level idea is to first derive an intermediary representation by reflecting $\vy$'s information to $\vx$. Then this intermediary information should also be used to collect $\vx$'s information into $\vy$. Ideally, such a back-and-forth should run infinitely long in order to properly reflect the mirrored information in both inputs. We define the stopping point of this refinement to be the fixed point of this iteration. %

To formally describe this setting, let $F_\theta$ and $G_\phi$ be the two feature extractors for the two inputs $\vx$ and $\vy$, respectively parameterized by sets of parameters $\theta$ and $\phi$. Each feature extractor will produce a hidden state that summarizes the information contained in the two inputs. 
More precisely, $F_\theta:\mathcal{X}\times\mathbb{R}^d\times\mathbb{R}^d\mapsto\mathbb{R}^d$ is a function that takes $\vx$, its own hidden state, and the other hidden state as input, and produces an updated hidden state $\vz_x\in\mathbb{R}^d$. This other given hidden state is meant to represent the summary of $\vy$ given $\vx$. %
Conversely, $G_\phi:\mathcal{Y}\times\mathbb{R}^d\times\mathbb{R}^d\mapsto\mathbb{R}^d$ achieves the same effect on $\vy$, providing the necessary input $\vz_y$ to $F_\theta$ (Fig.~\ref{fig:main1}).\\
\begin{figure}[t]
    \centering
    \begin{subfigure}{0.5\textwidth}
    \centering  
    \includegraphics[width=0.7\linewidth]{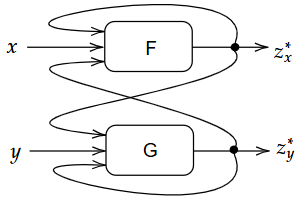}
    \caption{Overall diagram}
    \label{fig:main1}
    \end{subfigure}%
    \begin{subfigure}{0.5\textwidth}
    \centering  
    \includegraphics[scale=0.6]{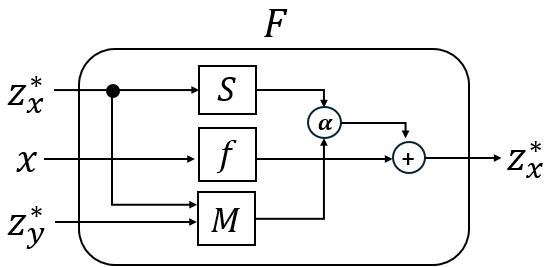} %
    \caption{Template for $F(x, \vz_x^*, \vz_y^*)$.}
    \label{fig:main2}
    \end{subfigure}
    \caption{Diagrams for the main architecture. Left is the overall architecture, and the right is the schematic diagram for $F$ ($G$ is symmetric). }
    \label{fig:diag}
\end{figure}
Eventually, our proposed model will take a mutually recursive form where the output of one component is fed back into the other, and vice versa. When viewed as an iterative process, one could define a single step of update as follows (left):
\begin{equation}  %
\begin{cases}
    \vz_x^{(k+1)}\gets F_\theta(\vx, \vz_x^{(k)}, \vz_y^{(k)}) \\
    \vz_y^{(k+1)}\gets G_\phi(\vy, \vz_y^{(k)}, \vz_x^{(k)})
\end{cases}
~~~\operatorname*{\Rightarrow}_{k\rightarrow\infty}~~~
\begin{cases}
    \vz_x^{*}=F_\theta(\vx, \vz_x^{*}, \vz_y^{*}) \\
    \vz_y^{*}=G_\phi(\vy, \vz_y^{*}, \vz_x^{*})
\end{cases}
    \label{eq:model}
\end{equation}
However, as the right-hand side of the above equation states, the ultimate outcome are the fixed points of such iterations. We compactly represent this in terms of the joint operator $T\triangleq[F_\theta; G_\phi]$: 
\begin{equation}
    \vz^* = T(\vu, \vz^*),~~~\text{where }~\vu\triangleq[\vx;\vy]~\text{ and } \vz^*\triangleq[\vz^*_x;\vz^*_y].
    \label{eq:compact}
\end{equation}
The Jacobian $J^*_T$ of $T$ at $\vz^*$ takes a block form as follows.
\begin{equation}
    \frac{\partial J}{\partial \vz^*}\triangleq J^*_T=\begin{pmatrix}
    \partial F/\partial \vz_x^* & \partial F/\partial \vz_y^*\\
    \partial G/\partial \vz_x^* & \partial G/\partial \vz_y^*
    \end{pmatrix}\triangleq
    \begin{pmatrix}
        \fzx & \fzy \\
        \gzx & \gzy
    \end{pmatrix}.
    \label{eq:block}
\end{equation}

While the right-hand side of Equation~\ref{eq:model} suggests that a \textit{fixed point iteration} (FPI) can be used to find such fixed points, it is common to rely on solver-based fixed point computation such as Anderson acceleration~\citep{anderson} or Broyden's algorithm~\citep{broyden}. %
{In practice we use a damped FPI truncated at a fixed number of steps $K$, and read the representation at $\vz^{(K)}$ rather than at a root. %
The reason for this is due to the nature of our approach: Being a coupled dynamic system, using solver-based approaches in our setting can fail numerically due to sensitivity with respect to hyperparameters. As such, we rely on a more robust FPI, combined with a residual penalty term described in Sec.~\ref{sec:obj} (also see Tbl.~\ref{tab:nofp}).

Notice that Equation~\ref{eq:compact} is precisely the formulation given in DEQ~\citep{deq}. Because our model is a special case of DEQ, it also inherits the theoretical guarantees, such as gradient computation and universality. %
However, ours differs from DEQ in how the equilibrium arises: DEQ is essentially a single-step computation,
whereas our model computes the fixed points by means of \textit{interaction} between the two inputs. The presence of such interactions requires us to carefully design the underlying architecture (Sec.~\ref{sec:design}).\\
Another key characteristic of our model is that it produces a pair of \emph{coupled} representations. That is, the $\vz_x$ and $\vz_y$ do not stand independently but rather jointly, each containing information from the other modality. We hypothesize that such paired embeddings will allow us to better represent the complementary information present in both inputs by iteratively referring to and updating each other's beliefs. 
\subsection{Theoretical Properties and Design Choices}
\label{sec:design}
In this section, we lay out theory-inspired design details for $F_\theta$ and $G_\phi$.
We first state a simple lemma that will assist our analysis.
\begin{mylem}
Given Equation~\ref{eq:model}, the Jacobians of $\vz_x^*$ and $\vz_y^*$ are given as
\begin{align*}
    \frac{d\vz_x^*}{d(\cdot)}&=\left(I-\fzx-\fzy\left(I-\gzy\right)^{-1}\gzx\right)^{-1}\left(\frac{\partial F}{\partial (\cdot)}+\fzy\left(I-\gzy\right)^{-1}\frac{\partial G}{\partial (\cdot)}\right)\nonumber\\
    \frac{d\vz_y^*}{d(\cdot)}&=\left(I-\gzy-\gzx\left(I-\fzx\right)^{-1}\fzy\right)^{-1}\left(\frac{\partial G}{\partial (\cdot)}+\gzx\left(I-\fzx\right)^{-1}\frac{\partial F}{\partial (\cdot)}\right),
\end{align*}
where $(\cdot)$ is a placeholder indicating any independent variable (e.g., $\theta$ or $\vx$).
\label{eq:lem1}
\end{mylem}
\begin{proof}
    Simple extension of the one given in~\citet{deq}. See Appendix~\ref{app:proof}.
\end{proof}
\textcolor{black}{The Lemma gives the gradient of the equilibrium with respect to the parameters:}
\begin{equation}
    \frac{\partial \mathcal{L}}{\partial \Theta}=\frac{\partial \mathcal{L}}{\partial \vz_x^*}\frac{d\vz_x^*}{d\Theta}+\frac{\partial \mathcal{L}}{\partial \vz_y^*}\frac{d\vz_y^*}{d\Theta},
    \label{eq:train}
\end{equation}
where $\mathcal{L}=\mathcal{L}(\mathcal{D};\Theta)$ is the loss function involving the parameter set $\Theta=\{\theta,\phi\}$ and the dataset $\mathcal{D}$ (Sec.~\ref{sec:obj}).\\
\textcolor{black}{In practice gradients are obtained by backpropagating through the $K$ unrolled updates rather than by an implicit solver, so Lemma~\ref{eq:lem1} is the analytical tool used above rather than the implemented gradient path (Appendix~\ref{app:exp_setup}).}\\
But Lemma~\ref{eq:lem1} can also be used to analyze how and when a failure occurs in the proposed architecture, especially by means of Jacobian analysis~\citep{imp_train,jacc_eq}. Viewing the fixed points as a function of the original inputs, we can interpret the lemma as stating how much raw information is reflected in the fixed point. Of particular interest is the pair of inverse terms appearing in both equations. These inverse terms are independent of the placeholder $(\cdot)$, and they capture the total amount of informational influence being passed between the two modalities. We thus call these inverses the \emph{total influence terms}. Please see Appendix~\ref{app:influence} for a detailed analysis of this term.

\paragraph{Sensitivity} Another aspect of the importance of the influence term is related to the sensitivity of the fixed points with respect to the raw inputs.
That is, following the compact notation given by Eqn.~\ref{eq:compact}, we look at the amount of perturbation $\epsilon_\vz$ of the fixed point $\vz^*$ when the input $\vu=[\vx;\vy]$ gets perturbed by a small amount $\epsilon_\vu$. Standard Taylor's expansion yields the following approximate bound (derivation in Appendix~\ref{app:deriv}):
\begin{equation*}
    \Vert\epsilon_\vz\Vert\lesssim\left\Vert\left(I-J_T^*\right)^{-1}\right\Vert\left\Vert\frac{\partial T}{\partial \vu}\right\Vert\Vert\epsilon_\vu\Vert.
\end{equation*}
The inverse term on the right hand side is the gathering of the total influence terms appearing in Lemma~\ref{eq:lem1}. This also signifies how the overall influence affects the sensitivity of the fixed point. An important consequence of this derivation is that the smaller the spectral radius $\rho(J_T^*)$, the more stable the fixed points are. This fact suggests two possible design choices: (1) structurally constrain the operator $T$ to have $\rho(J_T^*)$ less than 1, or (2) incorporate the spectral radius into the penalty term in the main objective.\\
One obvious way to ensure the first approach is to let $T$ be contractive, or 1-Lipschitz~\citep{1lexpl,approx_1nn} in the hidden states\footnote{A \emph{contractive map} $f$ satisfies $\Vert f(x)-f(y)\Vert<\Vert x-y\Vert$ for all possible $x$ and $y$ in $f$'s domain.}. However, while certain classes of 1-Lipschitz neural networks are proven to be universal function approximators~\citep{groupsort}, making a function globally contractive could be detrimental in our case (see `Fixed point collapse' below). Hence, we opt for the second option of regularization.
\paragraph{Cross-modal sensitivity} We can also analyze how sensitive the fixed points are with respect to the inputs $\vx$ and $\vy$. In particular, we analyze $d\vz_x^*/d\vy$ and $d\vz_y^*/d\vx$, which reflect how much of the raw inputs are reflected in the other embedding. By plugging in $\vx$ and $\vy$ into Lemma~\ref{eq:lem1}, we get
\begin{align}
    \frac{d\vz_x^*}{dy}&= \left(I-\fzx-\fzy(I-\gzy)^{-1}\gzx\right)^{-1}\fzy(I-\gzy)^{-1}\frac{\partial G}{\partial y}\nonumber\\
    \frac{d\vz_y^*}{dx}&= \left(I-\gzy-\gzx(I-\fzx)^{-1}\fzy\right)^{-1}\gzx(I-\fzx)^{-1}\frac{\partial F}{\partial x}.
    \label{eq:crosssensitivity}
\end{align}
Assuming that the inverses exist, the fact that these Jacobians are close to zero means that the fixed points computed have been detached from the other raw input. In order to prevent this failure mode, $\fzy$ and $\partial G/\partial y$ (resp., $\gzx$ and $\partial F/\partial x$) should be non-zero. This immediately suggests two conditions: (1) $F$ and $G$ should be using $\vz_y^*$ and $\vz_x^*$, respectively, non-trivially, and (2) $F$ and $G$ should be using $x$ and $y$, respectively, non-trivially. That is, both the cross-modal hidden states and raw inputs must be used in the other functions.
We thus propose to use the following templates for $F$ and $G$ (Fig.~\ref{fig:main2}):
\begin{align}
    F(x, \vz_x^*, \vz_y^*;\alpha)&=\alpha S(\vz_x^*)+(1-\alpha)M(\vz_x^*, \vz_y^*)+f(x)\nonumber\\
    G(y, \vz_y^*, \vz_x^*;\alpha)&=\alpha S(\vz_y^*)+(1-\alpha)M(\vz_y^*, \vz_x^*)+g(y),
    \label{eq:template}
\end{align}
where $\alpha$ is a learned mixing parameter and $f(\cdot), g(\cdot)$ are raw input feature extractors for $x$ and $y$, respectively. $S$ is the self-influence function, and $M$ is the cross-modal influence function. These are any differentiable functions that extract uni- and bimodal information from the hidden states.\\ 
Such a design is inspired by how the total influence decomposes into sums of self- and cross-modal influence terms. Even though it has no direct theoretical consequences, having both types of influences in $F$ and $G$ will ensure the existence of the Jacobians in Equation~\ref{eq:sum}. The $f$ and $g$ terms added in Equation~\ref{eq:template} are to ensure the non-zero cross-sensitivity given by Equation~\ref{eq:crosssensitivity}. \textcolor{black}{In our vision-language experiments $S$ is a self-attention block and $M$ a cross-attention block in which the other modality supplies keys and values; on CMU-MOSEI both are replaced by a per-modality MLP over its own injection, its own state and the other states.} %

\paragraph{Fixed point collapse} Next, we look at a failure case of our model that we term \emph{fixed point collapse}. This failure happens when both fixed points $\vz_x^*$ and $\vz_y^*$ get mapped to a single constant regardless of $\vx$ and $\vy$. An analytical way of seeing this is when the norms of the Jacobians $d\vz_x^*/dx$ and $d\vz_y^*/dy$ given by Lemma~\ref{eq:lem1} become close to 0. This is because the functions $F$ and $G$ can be thought of as being parameterized by $x$ and $y$, respectively, and the fixed points should be determined by the input values under normal circumstances. If not, the fixed points are independent of the inputs, hence fixed constants. \\
A clear case of this happens when the sub-Jacobians $\partial F/\partial x$ and $\partial G/\partial y$ are zero -- i.e., the functions $F$ and $G$ do not use the raw inputs. So it is clear that the two modules must use the inputs in a non-trivial way. In conjunction with these sub-Jacobian norms being close to zero, a collapse happens if the norm of the total influence is very small at the same time. Indeed, the norm of the input sensitivity, $d\vz_x^*/d\vx$, is bounded as follows (Let $(\cdot)=\vx$ in Lemma~\ref{eq:lem1}):
\begin{equation}
    \left\Vert\frac{d\vz_x^*}{d\vx}\right\Vert \leq \left\Vert\left(I-\fzx-\fzy(I-\gzy)^{-1}\gzx\right)^{-1}\right\Vert\left\Vert\frac{\partial F}{\partial \vx}\right\Vert
    \label{eq:selfsens}
\end{equation}
So even if $\partial F/\partial \vx$ is non-zero (but small), a fixed point collapse effect can still happen if the norm of the total influence is small too. Note that this partially contrasts the suggestions made in the sensitivity discussion. In order to minimize the perturbation error, it was advised that $\rho(J_T^*)$ be as small as possible. However, it also turns out that having too small a $\rho(J_T^*)$ can induce a near-fixed point collapse, albeit conditioned on the small sub-Jacobian norm. One way to view this is: If the total influence is too large, the model will be overflown with influence making the fixed points too sensitive. If it is too small, then there will not be enough momentum to produce meaningful fixed points.

\subsection{Training Objective}
\label{sec:obj}
The main training objective of \modelname~depends on the task on which it is being trained. In addition to the main objective $\mathcal{L}_{task}$, we add two regularization terms to address the theoretical issues raised.

In order to balance expressivity and collapse prevention, the spectral radius $\rho(J_T^*)$ needs to be smaller than 1, but not too small. Although this is conditioned on the $\partial J/\partial\vx$ being small too, we propose to use a range penalty on $\rho(J_T^*)$ around the fixed point as a safety measure.
\begin{equation*}
    \mathcal{L}_{jac}=\left(\max\left\{0, \Vert\hat{J}_T\Vert-\rho_h\right\}\right)^2+\left(\max\left\{0, \rho_\ell-\Vert\hat{J}_T\Vert\right\}\right)^2,
\end{equation*}
where $\rho_h$ and $\rho_\ell$ are the hyperparameters used to put the Jacobian norm inside the range $[\rho_\ell,\rho_h]$. Empirically, this also has the effect of making the computation deeper. That is, too small a spectral radius will make the fixed points too easy to reach, which in turn loses the benefits of iterative refinement. Forcing a lower bound on the radius will result in a more expressive function.\\
We use the estimate $\Vert\hat{J}_T\Vert$ of the true Jacobian norm around the fixed point $\vz^*$, since full computation is expensive. For example, we use the Hutchinson estimator~\citep{hutch} in our experiments.

The second regularization term we propose is called the \textit{fixed point correction} term. For that, we first define the fixed point residual given the input $\vu=[\vx;\vy]$ as follows:
\begin{equation*}
    r(\vz^*;\vu)\triangleq \vz^*-T(\vu, \vz^*)
\end{equation*}
Then we define the fixed point correction regularizer as the sum of
\begin{equation*}
    \mathcal{L}_{fpc}=\frac{1}{2}\sum_{m\in\{\vx,\vy\}}\frac{\Vert r(\vz^*;\vu)_m\Vert}{{\Vert\vz^*_m\Vert}},    
\end{equation*}
where the subscript $m$ indexes one of the two input modalities. While $r(\vz^*;\vu)$ should be zero in theory, practical limitations or settings often make it nonzero. %
{For instance, a fixed iteration budget $K$ leaves a nonzero residual whenever the map is slow to settle, so the representation we read is only an approximate fixed point.}
High value of $\mathcal{L}_{fpc}$ indicates those difficult cases. Penalizing the model using this term will update the model (e.g., by flattening) to yield fixed points that can be reached through small number of iterations.\\
Empirically this term is what keeps the truncated iterate close to self-consistent%
(Table~\ref{tab:nofp} in Appendix~\ref{app:exp_analysis}).
The final objective we propose is thus a weighted sum of these regularizers plus the main task loss.
\begin{equation}
    \mathcal{L}(\mathcal{D};\Theta)=\mathcal{L}_{task}+\lambda_j\mathcal{L}_{jac}+\lambda_f\mathcal{L}_{fpc}
    \label{eq:loss}
\end{equation}

\section{Experiments}
\label{sec:experiments}
The central goal of the experiments presented here is to demonstrate the value of mutual iterative refinement of features. That is, whether the prediction can be improved or recovered from the initial faulty guess through repeated refinement, rather than achieving state-of-the-art performance. We first describe the experimental setups chosen to verify this property. For ease of presentation, we denote our proposed model as \modelname{} (\textbf{M}utual \textbf{EQ}uilibrium model). Please refer to Appendix~\ref{app:exp_setup} for details regarding implementation and experimental setups. Appendices~\ref{app:extra_task} and~\ref{app:exp_analysis} contain more experimental results and analyses.
\paragraph{Baselines} While we acknowledge there are many multimodal fusion techniques, we rely on simpler baselines in order to eliminate the extra architectural innovations and isolate the effects of the proposed structure. The baseline we compare against are simple concatenation-based fusion and attention-based fusion. The attention-based fusion approach is further divided into two variants: (1) perform self attention on the concatenated inputs $[\vx;\vy]$, and (2) perform cross-attention on the two inputs. In comparison, our model produces a pair of vectors $\vz_x^*, \vz_y^*$, which are concatenated to form the final feature. All attention models were adjusted to match the parameter count to that of our model.
\paragraph{Datasets}
We evaluate on the following benchmarks spanning different tasks and modalities: Hateful Memes~\citep{hateful}, VCR~\citep{vcr}, VQA v2 \citep{vqa_v2}, SNLI-VE \citep{snli}, and CMU-MOSEI~\citep{mosei}. Please refer to Appendix~\ref{app:exp_setup} for detailed statistics of these datasets.
\subsection{Accuracy Results}

We first verify the usability of \modelname{} in the prediction domain. Our \modelname{} improves over simple concatenation on four of the five benchmarks (Table~\ref{tab:main}). The largest gain is on VQA v2 ($+14.36$\,pp), and the exception is CMU-MOSEI, where the two are within seed noise of each other ($-0.59$\,pp against a two-sigma spread of $1.10$\,pp). Against the stronger of the two parameter-matched attention baselines, \modelname{} leads on SNLI-VE, VCR and VQA v2 by smaller margins ($+0.74$ to $+0.94$\,pp), and trails on Hateful Memes ($-2.31$\,pp). Regarding the CMU-MOSEI results \textcolor{black}{being indistinguishable from both baselines, further examination showed that this is due to the characteristics of the dataset itself, rather than the structural design choices (Sec.~\ref{sec:ablation}).
}

\begin{table}[h]
\centering
\caption{Accuracy results across multimodal benchmarks. All numbers are means over three seeds under an identical training recipe. Best attn. is the stronger of parameter-matched cross- and self-attention. \textcolor{black}{ VQA~v2 and VCR use non-standard evaluation subsets; see Appendix~\ref{app:exp_setup}.}}
\label{tab:main}
\begin{tabular}{lccccc}
\toprule
Dataset & Metric & Concat & {Best attn.} & \modelname & $\Delta$ concat / {attn.} \\
\midrule
Hateful Memes & AUROC & 0.6873 & {0.7251} & 0.7021 & +1.48 / {$-$2.31}pp \\
VCR & Acc & 56.38 & {62.09} & \textbf{62.95} & +6.57 / {+0.85}pp \\
VQA v2 & Acc & 48.97 & {62.39} & \textbf{63.33} & +14.36 / {+0.94}pp \\
SNLI-VE & Acc & 71.61 & {74.40} & \textbf{75.15} & +3.53 / {+0.74}pp \\
CMU-MOSEI & Acc-7 & 54.79 & {54.74} & 54.20 & $-$0.59 / {$-$0.54}pp \\
\bottomrule
\end{tabular}
\end{table}
While \modelname{} outperforms simple concatenation, there seems to be little difference between the attention-based baselines. We conjecture this is because of how attention mechanisms are already powerful modality fusion algorithms~\citep{attn_strong,mm_transformer_survey}. Adding to that is the fact that the CLIP and BERT encoders for both modalities are state-of-the-art. This is somewhat expected, since the strength of our model is not in modality fusion but in iterated reasoning. We show this evidence in the next section. \textcolor{black}{Iterating to the read depth costs $2.13\times$ the wall clock of the fastest single-pass baseline, while peak memory does not grow with $K$ (Table~\ref{tab:cost}).}

\subsection{Iterative Improvement Results}
In this section, we verify how \modelname{} can improve or update the performance as it performs iteration. As the first example, we visualize  how the visual grounding changes along with the fixed point iteration progress. Specifically, we test our model on the COD10K~\citep{cod10k} dataset, where the task is to locate various real-world objects that lay hidden against natural backgrounds. We first train our \modelname{} model on the SNLI-VE dataset. Then at inference time, we feed the raw image as the $\vx$ component, and the paired $\vy$ component is a text of the form \emph{``A \textnormal{[NAME]} is hidden in the picture''}, where `NAME' is replaced with the ground truth label. Then we see which part of the image is focused at as the fixed point iteration goes on. Figure~\ref{fig:qual} shows the Grad-CAM~\citep{cam} visualization of the object at different iterations. 
\begin{figure}
    \centering
    \includegraphics[width=\linewidth]{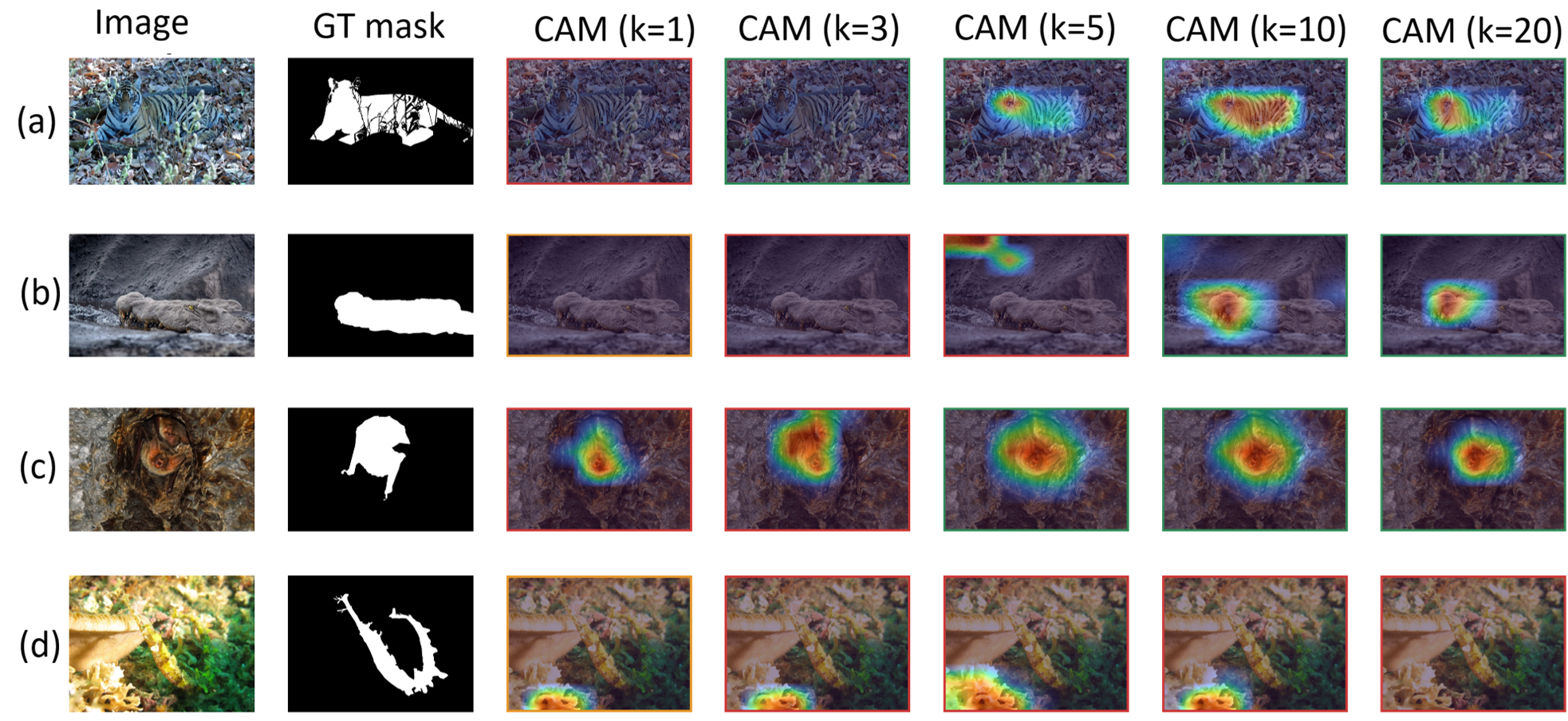}
    \caption{Application of \modelname{} (trained on SNLI-VE) to camouflaged-object images from COD10k, where ground-truth masks allow direct verification of visual grounding.  Border colors give the prediction at each step (green: entailment, red: contradiction, orange: neutral). Best viewed in color.}
    \label{fig:qual}
\end{figure}
On the success cases (top three rows), attention focuses on the hidden objects within the first few iterations, while in the failure case (bottom row) never localizes.\footnote{Results of training and testing on it directly are in Table~\ref{tab:cod} (Appendix~\ref{app:extra_task}).}

In addition to the heat maps, we also show how the prediction probability changes. The accuracy vs. $k$ plot is shown in Figure~\ref{fig:acc_vs_k}:
The probability assigned to entailment for the matched caption (solid) and for a mismatched caption (dotted), over $k \in [1, 20]$ with the read depth $K{=}10$ marked. Dataset-level accuracy peaks near $K$ (80.0 at $k{=}10$ vs.\ 77.2 at $k{=}300$). 
\begin{figure}[ht]
  \centering  
  \begin{subfigure}[b]{0.22\textwidth}
    \centering
    \includegraphics[width=\textwidth]{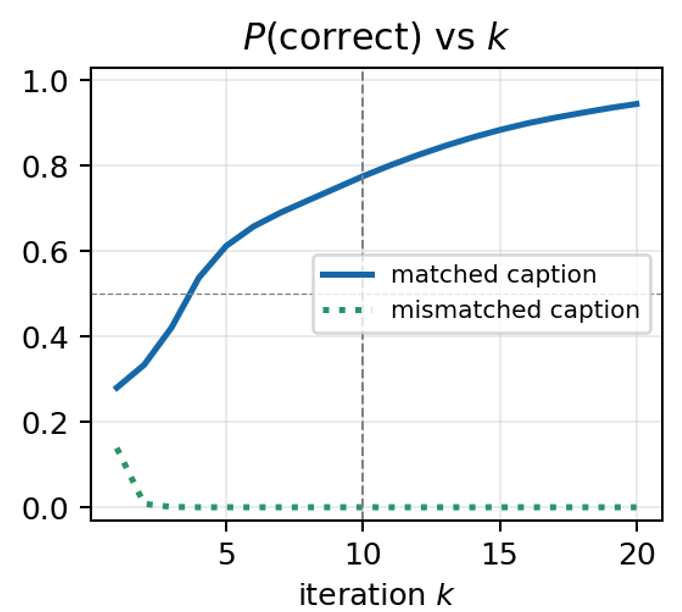}
    \caption{Tiger (success)}
    \label{fig:tiger}
  \end{subfigure}
  \hfill %
  \begin{subfigure}[b]{0.22\textwidth}
    \centering
    \includegraphics[width=\textwidth]{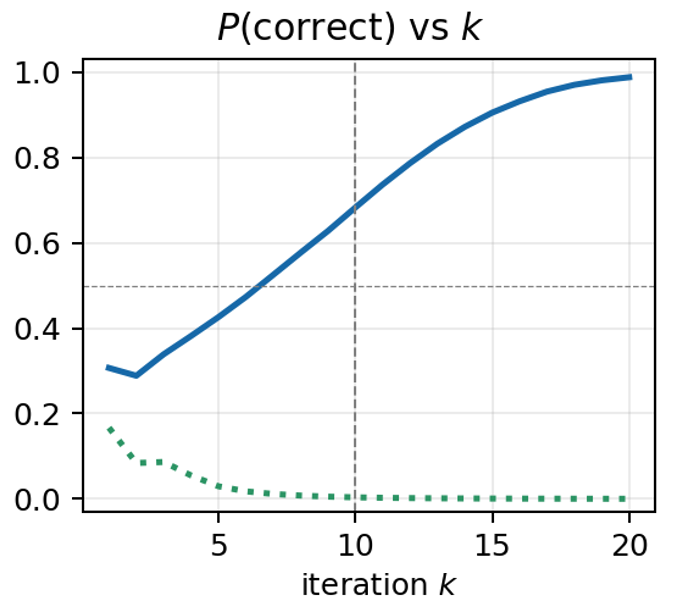}
    \caption{Crocodile (success)}
    \label{fig:croc}
  \end{subfigure}
    \hfill
  \begin{subfigure}[b]{0.22\textwidth}
    \centering
    \includegraphics[width=\textwidth]{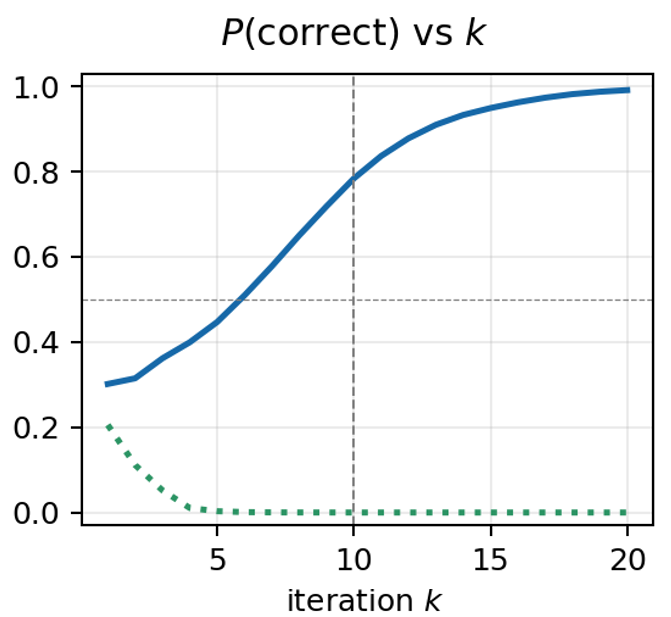}
    \caption{Bat (success)}
    \label{fig:bat}
  \end{subfigure}
  \hfill
  \begin{subfigure}[b]{0.22\textwidth}
    \centering
    \includegraphics[width=\textwidth]{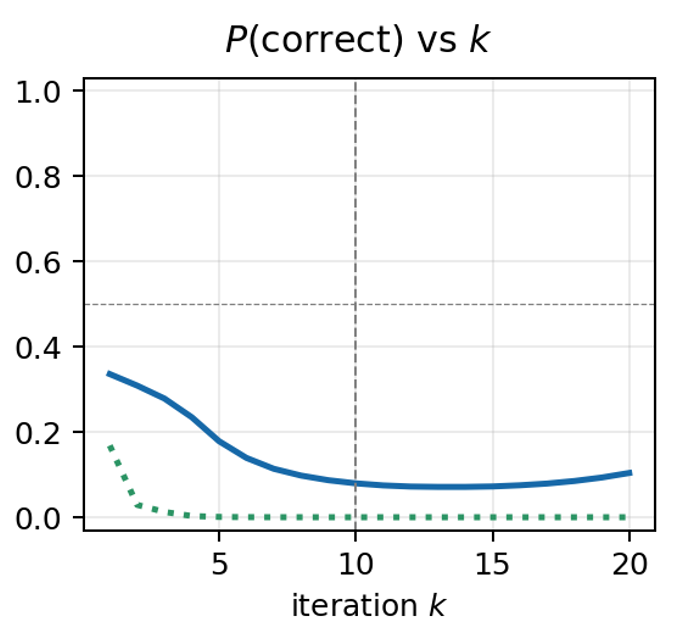}
    \caption{Pipefish (failure)}
    \label{fig:ghost}
  \end{subfigure}

  \caption{Plots of $P(\text{correct})$ vs. $k$.}
  \label{fig:acc_vs_k}
\end{figure}

The dotted green curves on the plots are the accuracy change when the image is paired with a completely irrelevant text (e.g., ``A man is walking down the street.''). 
Here, we can observe that the model increases (and decreases) the correct (and incorrect, resp.) probabilities for the success cases.

\subsection{Ablation Studies}
\label{sec:ablation}
We ablate each component on VQA v2, the benchmark where \modelname's gain over concatenation is largest, so the ablation measures the components where coupling demonstrably pays (Table~\ref{tab:ablation}).  Removing cross-attention costs $4.20$\,pp, ten times its spread, and reducing to a single iteration also degrades significantly, so both the cross-modal path and the iteration itself contribute where coupling pays. The retrained full model ($63.11$) agrees with Table~\ref{tab:main} ($63.33$) within seed noise.

\paragraph{{On Soft Gating}}
Removing the soft gate costs only $0.54$\,pp on VQA v2 (within the paired spread, Table~\ref{tab:ablation}), so the gate is not what makes the model work.
It is, however, where the model can fail. On VQA v2 the learned gate collapsed to approximately $10^{-7}$ on real inputs, which makes the update $F(\vx,\vz_x,\vz_y)=\vz_x$ and drives both cross-modal Jacobians in Equation~\ref{eq:crosssensitivity} to zero. This is the failure mode Section~\ref{sec:design} identifies, observed in training.

\begin{table}[thb]
\centering
\caption{Component ablation on VQA v2}%
\label{tab:ablation}
\begin{tabular}{lcc}
\toprule
Configuration & Acc (\%) & $\Delta$ \\
\midrule
Full \modelname & {63.11} & -- \\
w/o Soft Gating & {62.57} & {$-$0.54} \\
w/o Cross-Attention & {\textbf{58.90}} & {$-$4.20} \\
w/o Self-Attention & {62.33} & {$-$0.78} \\
Single Iteration & {62.39} & {$-$0.71} \\
\bottomrule
\end{tabular}
\end{table}
\paragraph{Is the mixer the point?}
We replaced the cross-modal function $M$ in Equation~\ref{eq:template} with a gMLP~\citep{gmlp} block, within two percent in parameter count ($34.14$M against $33.47$M), and retrained on SNLI-VE. The two are indistinguishable: $75.26\%$ against $75.15\%$, a seed-matched difference of $+0.11$\,pp against a two-sigma spread of $0.25$\,pp. We also asked whether mixing again after the iteration helps, by replacing the pooled-concatenation readout with a symmetric cross-attention merge, against a capacity-matched control with the same parameter budget and no cross-modal path ($38.196$M against $38.195$M). The difference between them is $+0.01$\,pp. Neither the choice of mixer nor additional mixing after the iteration accounts for the gain. This seems to hint at the architectural contribution to the current level of performance, but more study is needed to either confirm or refute this claim. We leave it as future work to further verify this.

\paragraph{Is it the two states?}
\textcolor{black}{Two parameter-matched baselines replace only the state layout: one joint state that still carries tokens, and the feature-sum design of \citet{deqmm}, which pools each modality to a vector. \modelname{} leads both on VQA~v2, by $0.35$\,pp and $11.39$\,pp against paired $2\sigma$ of $0.17$ and $1.34$\,pp. CMU-MOSEI isolates the state count, since its state is already one vector per modality: collapsing the three costs $0.84$\,pp against a paired $2\sigma$ of $0.22$\,pp, even there, where \modelname{} does not beat its fusion baselines. On Hateful Memes nothing separates (Table~\ref{tab:statedesign}).} See Appendix~\ref{app:exp_analysis} (`State Design' paragraph) for more discussions on this matter.

\paragraph{\textcolor{black}{Do these benchmarks need both modalities?}}
CMU-MOSEI is the one benchmark where \modelname{} does not improve on its baselines, so we asked whether it rewards fusion at all. Retraining each model from scratch with one modality zeroed at its encoder output gives Table~\ref{tab:modality}. On CMU-MOSEI a model that can see neither audio nor video is indistinguishable from the full model and from both baselines, and removing the cross-modal path from the update there costs $0.13$\,pp against a paired $2\sigma$ of $0.97$\,pp, where the same removal costs $4.20$\,pp on VQA~v2. We therefore read the CMU-MOSEI result as a property of that benchmark under our feature pipeline rather than of the architecture. The same control separates cleanly on Hateful Memes and on SNLI-VE, so it is capable of detecting a modality that matters. SNLI-VE is worth a second look: its text-only model still reaches $70.42$, far above the three-way chance rate, which is the annotation artifact SNLI is known for. On top of that artifact the image is worth $4.73$ points to \modelname{} and only $1.19$ to concatenation, so most of the $3.53$ point gap between them in Table~\ref{tab:main} is a difference in how much of the image each model uses.
\begin{table}[htb]
\centering
\caption{Single-modality controls. Each entry is a model retrained from scratch with one modality zeroed at the encoder output, means over three seeds. Metrics differ per benchmark, so only the verdicts are comparable across rows, not the sizes of the gaps. CMU-MOSEI has three modalities, so `text only' there means audio and video both zeroed and an image-only arm does not apply.}
\label{tab:modality}
\begin{tabular}{lcccc}
\toprule
Dataset & Metric & Both & Text only & Image only \\
\midrule
SNLI-VE       & Acc   & 75.15  & 70.42  & 33.85 \\
Hateful Memes & AUROC & 0.7080 & 0.6330 & 0.6526 \\
CMU-MOSEI     & Acc-7 & 54.70  & 53.97  & -- \\
\bottomrule
\end{tabular}
\end{table}
\vspace{-1em}

\section{Conclusion and Future Works}
This work presents a new architecture based on mutual feedback to refine bimodal representations. Although we have mainly demonstrated our algorithm in the context of classification by fusion in the experiments, we would like to make it clear that the true contribution of this work goes beyond fusion. The real value of \modelname{} is in the iterative nature of information refinement while still offering performance better than, or on par with baselines. Under this viewpoint, we suggest several promising future directions:
\begin{itemize}
    \item Multi-modal generalization: How to incorporate more than 2 modalities (Appendix~\ref{app:multmod}).
    \item Modular LLMs: Where multiple LLMs with different specializations iteratively refine each other's outputs.
    \item Agent agreement: Multiple agents converse to iteratively resolve discrepancy in observation.
    \item Iteration benefits: When and what problems benefit from \modelname{}?
\end{itemize}
In all of these cases, the key point that needs to be investigated is the theoretical foundation of iterative refinement under the given constraints.

Furthermore, the fact that \modelname{} is essentially a coupled dynamical system calls for closer investigation into the theoretical connections to traditional dynamical systems (e.g.,~\cite{smallgain}). A practical application in the same direction can also be found in Physics-Informed Neural Networks~\citep{pinn}, where PDEs that describe such physical problems are in fact dynamical systems.

\subsection*{AI Disclosure}
In this work, we used generative AI tools for editing software code used in the experiments and relevant literature search. We have not used generative AI tools for any other tasks. All authors have reviewed all AI-assisted work. Specifically, the AI-generated code was reviewed and approved by all authors in-person, and the literature survey results were cross-checked as well. We take responsibility for the final content of this work, including text, claims or artifacts produced with the aid of generative AI.
\subsection*{Reproducibility Statement}
Every number in this paper regenerates from released run records: a bundle containing per-run configurations and results, model and training source, and scripts that rebuild each table byte-identically, with a verifier that checks file hashes and regenerates all tables from scratch. Checkpoints are indexed by SHA-256. The bundle will be released publicly with the final version of this paper.
\subsection*{Ethics Statement}
This work uses only datasets that are publicly available, and therefore does not require an IRB approval. No part of this work uses materials derived from human or animal subjects, either. Please see the `AI Disclosure' statement above for AI tool uses in this work.
\bibliography{refs}
\bibliographystyle{iclr2027_conference}
\newpage
\appendix
\section*{Appendix}
\section{Proofs and Derivations}
\subsection{Proof of Lemma~\ref{eq:lem1}}
\label{app:proof}
\textbf{Lemma~\ref{eq:lem1}.} 
\textit{Given Equation~\ref{eq:model}, the Jacobians of $\vz_x^*$ and $\vz_y^*$ are given as}
\begin{align*}
\frac{d\vz_x^*}{d(\cdot)}&=\left(I-\frac{\partial F}{\partial \vz_x^*}-\frac{\partial F}{\partial \vz_y^*}\left(I-\frac{\partial G}{\partial \vz_y^*}\right)^{-1}\frac{\partial G}{\partial\vz_x^*}\right)^{-1}\left(\frac{\partial F}{\partial (\cdot)}+\frac{\partial F}{\partial \vz_y^*}\left(I-\frac{\partial G}{\partial\vz_y^*}\right)^{-1}\frac{\partial G}{\partial (\cdot)}\right)\nonumber\\
    \frac{d\vz_y^*}{d(\cdot)}&=\left(I-\frac{\partial G}{\partial \vz_y^*}-\frac{\partial G}{\partial \vz_x^*}\left(I-\frac{\partial F}{\partial \vz_x^*}\right)^{-1}\frac{\partial F}{\partial\vz_y^*}\right)^{-1}\left(\frac{\partial G}{\partial (\cdot)}+\frac{\partial G}{\partial \vz_x^*}\left(I-\frac{\partial F}{\partial\vz_x^*}\right)^{-1}\frac{\partial F}{\partial (\cdot)}\right),
\end{align*}
\textit{where $(\cdot)$ is a placeholder indicating any independent variable (e.g., $\theta$ or $\vx$).}
\begin{proof}
For ease of presentation, we use the notation given in Equation~\ref{eq:block}.
We start with the original fixed point equation given in Equation~\ref{eq:model}. Applying the implicit function theorem on both fixed points w.r.t. the placeholder $(\cdot)$ yields the following coupled equations (parameter subscripts omitted).
    \begin{align}
        \frac{d\vz_x^*}{d(\cdot)}&=\frac{\partial F}{\partial (\cdot)}+\fzx\frac{d\vz_x^*}{d(\cdot)}+\fzy\frac{d\vz_y^*}{d(\cdot)}\label{eq:line1}\\
        \frac{d\vz_y^*}{d(\cdot)}&=\frac{\partial G}{\partial (\cdot)}+\gzy\frac{d\vz_y^*}{d(\cdot)}+\gzx\frac{d\vz_x^*}{d(\cdot)}.\label{eq:line2}
    \end{align}
We first tackle the $d\vz_x^*/d(\cdot)$ by grouping like terms:
\begin{equation}
    \frac{d\vz_x^*}{d(\cdot)}=(I-\fzx)^{-1}\left(\frac{\partial F}{\partial (\cdot)}+\fzy\frac{d\vz_y^*}{d(\cdot)}\right)
\end{equation}
Substitute this into Equation~\ref{eq:line2} to isolate the $d\vz_y^*/d(\cdot)$ terms:
\begin{align*}
    \frac{d\vz_y^*}{d(\cdot)}&=\frac{\partial G}{\partial (\cdot)}+\gzy\frac{d\vz_y^*}{d(\cdot)}+\gzx\left((I-\fzx)^{-1}\left(\frac{\partial F}{\partial (\cdot)}+\fzy\frac{d\vz_y^*}{d(\cdot)}\right)\right)\\
    &=\frac{\partial G}{\partial (\cdot)}+\gzy\frac{d\vz_y^*}{d(\cdot)}+\gzx(I-\fzx)^{-1}\frac{\partial F}{\partial(\cdot)}+\gzx(I-\fzx)^{-1}\fzy\frac{d\vz_y^*}{d(\cdot)}
\end{align*}
Solving for $d\vz_y^*/d(\cdot)$ yields the following.
\begin{align*}
    \left(I-\gzy-\gzx(I-\fzx)^{-1}\fzy\right)\frac{d\vz_y^*}{d(\cdot)}=\frac{\partial G}{\partial (\cdot)}+\gzx(I-\fzx)^{-1}\frac{\partial F}{\partial(\cdot)}\\
    \Rightarrow \frac{d\vz_y^*}{d(\cdot)}=\left(I-\gzy-\gzx(I-\fzx)^{-1}\fzy\right)^{-1}\left(\frac{\partial G}{\partial (\cdot)}+\gzx(I-\fzx)^{-1}\frac{\partial F}{\partial(\cdot)}\right)
\end{align*}
Solving for $d\vz_x^*/d(\cdot)$ is omitted due to symmetric arguments.

\end{proof}
Note that this derivation is assuming that $\vx$ and $\vy$ are independent of $(\cdot)$. If that is not the case, simply change $\partial F/\partial (\cdot)$ and $\partial G/\partial (\cdot)$ to 
\begin{equation*}
    \frac{\partial F}{\partial\vx}\frac{\partial\vx}{\partial(\cdot)}\text{~~~~~and~~~~~} \frac{\partial G}{\partial\vy}\frac{\partial\vy}{\partial(\cdot)}\text{~, respectively.}
\end{equation*}
\subsection{Influence dynamics} %
\label{app:influence}
Following the convention given in Equation~\ref{eq:block}, the Neumann series expansion of the influence of $d\vz_x^*/d(\cdot)$ yields the following (The other influence term is omitted due to symmetry). 
\begin{equation}
    \left(I-\fzx-\fzy(I-\gzy)^{-1}\gzx\right)^{-1}=\sum_{k=0}^\infty \left(\underbrace{\fzx}_{\text{Self-influence}}+\underbrace{\fzy(I-\gzy)^{-1}\gzx}_{\text{Cross-modal influence}}\right)^k%
    \label{eq:sum}
\end{equation}
Firstly, the self-influence term $\fzx$ indicates the direct influence of $\vz_x^*$. Next, the latter cross-modality influence term indicates the influence interaction between $F$ and $G$. To see why, notice the inner inverse term $(I-\gzy)^{-1}$ is also an infinite sum $\sum_k (\gzy)^k$, meaning the `local' influence of $\vz_y^*$ on $G$. It is being multiplied by $\fzy$ and $\gzx$ on both sides, which stand for single-step cross-modal influences. Put together, the term $\fzx+\fzy(I-\gzy)^{-1}\gzx$ can be loosely interpreted as the amount of influence $F$ processes on its own (term $\fzx$), \emph{and} the amount of information $F$ sends to $G$ (term $\gzx$), that influence being combined and looped to refine $\vz_y^*$ (term $(I-\gzy)^{-1}$), and finally being sent back to $F$ (term $\fzy$). Finally, this overall influence is being added over all possible powers (i.e., repetitions) of it. This means that the final inverse term is considering the information resulting from all possible influence paths.
\subsection{Sensitivity}
\label{app:deriv}
We wish to know how much the fixed point $\vz^*$ gets perturbed if we perturb the input $\vu=[\vx;\vy]$ by $\epsilon_\vu$.
To quantify this amount, we let $\vz_{\vu+\epsilon_\vu}^*\triangleq\vz^*+\epsilon_\vz$ (i.e., how much the fixed point of $\vu+\epsilon_\vu$ deviates from the unperturbed fixed point $\vz^*$.)
By Taylor expansion, we have
\begin{align*}
    T(\vz^*+\epsilon_\vz,\vu+\epsilon_\vu)&\approx T(\vz^*,\vu)+\frac{\partial T}{\partial \vz^*}\epsilon_\vz+\frac{\partial T}{\partial \vu}\epsilon_\vu\\
    \Rightarrow \vz^*+\epsilon_\vz&\approx \vz^*+\frac{\partial T}{\partial \vz^*}\epsilon_\vz+\frac{\partial T}{\partial\vu}\epsilon_\vu\\
    \Rightarrow\Vert\epsilon_\vz\Vert&\lesssim\left\Vert\left(I-\frac{\partial T}{\partial \vz^*}\right)^{-1}\right\Vert\left\Vert\frac{\partial T}{\partial \vu}\right\Vert\Vert\epsilon_\vu\Vert.
\end{align*}
\subsection{Sharper Condition for Invertibility}
\label{app:sharp}
To ensure the invertibility of the total influence, we proposed to constrain the norm of the Jacobian $J_T^*$. That itself is sufficient, but we can derive a sharper condition at the expense of more computation.\\
\begin{myprop}
If we have $\rho(\gzy)<1$ and $\rho(\fzx)<1$, the following is sufficient to guarantee the existence of the inverse term.
\begin{equation}
    \frac{\Vert\gzx\Vert\Vert\fzy\Vert}{(1-\Vert\gzy\Vert)(1-\Vert\fzx\Vert)}<1
    \label{eq:inv_cond}
\end{equation}
\end{myprop}
\begin{proof}
Assume $\Vert\gzy\Vert<a\leq1, \Vert\fzx\Vert<b\leq1$ and factor the total influence term as follows:
\begin{equation*}
    I-\fzx-\fzy(I-\gzy)^{-1}\gzx=(I-\fzx)(I-(I-\fzx)^{-1}\fzy(I-\gzy)^{-1}\gzx).
\end{equation*}
Then the sufficient condition for the inverse of the LHS to exist is 
\begin{equation}
    \rho((I-\fzx)^{-1}\fzy(I-\gzy)^{-1}\gzx)<1.
    \label{eq:eff}
\end{equation}
Instead of the spectral radius, which is hard to compute, let us focus on the norm since it forms an upper bound of $\rho$ (the spectral radius is upper-bounded by the norm).\\
By the Neumann sum definition, we have
\begin{align*}
    \left\Vert(I-\fzx)^{-1}\right\Vert\leq\sum_{k=0}^\infty\Vert\fzx\Vert^k<\frac{1}{1-b}\\
    \left\Vert(I-\gzy)^{-1}\right\Vert\leq\sum_{k=0}^\infty\Vert\gzy\Vert^k<\frac{1}{1-a},
\end{align*}
due to the norm assumptions made above.
Thus the bound of the norm of the LHS of Eqn.~\ref{eq:eff} is given as
\begin{equation}
    \left\Vert(I-\fzx)^{-1}\fzy(I-\gzy)^{-1}\gzx\right\Vert\leq\frac{\Vert\gzx\Vert\Vert\fzy\Vert}{(1-a)(1-b)}
    \label{eq:suff}
\end{equation}
Inequality~\ref{eq:suff} ensures that the influence term is invertible. This implies that Eqn.~\ref{eq:inv_cond} is sufficient condition for the invertibility of the influence term.
\end{proof}
Instead of bounding the norm of the entire $J_T^*$, this bound has the advantage that a fine-grained tuning is possible. 
\section{More than Two Modalities}
\label{app:multmod}
Although the main focus of this paper is to present bimodal feedback refinement due to  theoretical tightness, we also mention a way to generalize this approach to more than two modalities. We propose two possibilities that will be left as future works.
\paragraph{Hierarchical composition} Given modalities $\{\vx_i\}_{i=1}^N$ as the input set, we form a merge tree $\{(((\vx_i, \vx_j), \vx_k), \vx_\ell), \cdots, )\}$ that dictates the order of combination. The ordering can be determined randomly, or preferably by domain knowledge. Then for each pair $(\vx_i, \vx_j)$, the \modelname{} generates the pair $(\vz_i^*, \vz_j^*)$. In order to apply it to the next $\vx_k$, we need to combine the generated pair through another MLP $h:\mathbb{R}^d\times\mathbb{R}^d\mapsto\mathbb{R}^d$ to produce a temporary summary embedding $h(\vz_i^*, \vz_j^*)$. Then we take this summary and join with the next modality in line to produce the next embedding.\\
The shortcoming of this approach is obviously the order of mixing. Unless there is a justified way of ordering the modalities, random ordering will likely result in high-variance performance in the end. It also introduces an additional set of parameters for the combiner MLP.
\paragraph{Complete graph composition} When all modalities need to interact with everyone else, hence forming a complete graph-like connection topology. In this case, the individual encoding function is defined as
\begin{equation*}
    \vz_i^*=F_i(\vx_i, \{\vz_k\}_{k=1}^N).
\end{equation*}
This formulation is a natural extension of the idea given in this paper, and is intuitive to understand.
Unfortunately, the standard Jacobian analysis we used in our work no longer yields simple solutions under this scheme. The Jacobian $J_T^*$ now consists of $N^2$ block Jacobians, which means $(I-J_T^*)^{-1}$ no longer decomposes into nicely interpretable form given by Lemma~\ref{eq:lem1}. Hence, future works should further analyze the behavior of this larger Jacobian.
\section{Additional Experimental Results}
\subsection{Experimental Setups}
\label{app:exp_setup}
\paragraph{Dataset details} The following table shows the train-val-test split and class information in the datasets used. Following standard practice, we treat VQA v2 as classification over the 3,129 most frequent answers. Owing to compute constraints, we train and evaluate on the val2014 split, divided 9:1 into train and held-out portions with a fixed seed. The official train split is not used.  \textcolor{black}{The partition is over questions rather than images, and VQA~v2 carries $5.29$ questions per image on average, so the two halves share images: $99.7\%$ of held-out questions use an image that also appears in training. The image encoder is frozen, so this exposure reaches the model only through the fusion block, but the VQA~v2 numbers are best read as a held-out-question rather than a held-out-image estimate.}
\begin{table}[h]
\centering
\caption{Dataset statistics.}
\label{tab:datasets}
\begin{tabular}{llcccc}
\toprule
Dataset & Task & Train & Val & Test & Classes \\
\midrule
Hateful Memes~\citep{hateful} & Binary classification & 8,500 & 500 & 1,000 & 2 \\
VCR~\citep{vcr} & 4-way multiple choice & 212K & 26K & 25K & 4 \\
VQA v2 \citep{vqa_v2} & Open-ended QA & 193K & 21K & -- & 3,129 \\
SNLI-VE \citep{snli} & Visual entailment & 533K & 9.6K & 9.6K & 3 \\
CMU-MOSEI~\citep{mosei} & Sentiment (7-class) & 16K & 2K & 5K & 7 \\
\bottomrule
\end{tabular}
\end{table}

\paragraph{Implementation details}
We use a frozen CLIP~\citep{clip} image encoder and a frozen BERT-base~\citep{bert} text encoder (per-benchmark variations in Table~\ref{tab:encoders}).
The \modelname{} module has 768-dimensional hidden states with 8 attention heads.
We iterate a damped fixed-point update for $K=10$ steps, \textcolor{black}{$\vz^{(k+1)}=\beta\,T(\vu,\vz^{(k)})+(1-\beta)\,\vz^{(k)}$}, and train with the Jacobian range penalty of Section~\ref{sec:obj} ($\rho_\ell=0.7$, $\rho_h=0.9$, Hutchinson estimate; $\beta$ and probe counts in Table~\ref{tab:hyper}). \textcolor{black}{The two mixing coefficients act at different levels. The mixing parameter $\alpha$ \textcolor{black}{of the template} in Equation~\ref{eq:template} weighs the self-influence $S$ against the cross-modal $M$  \textcolor{black}{within a block}, while $\beta$ damps the iteration of the joint operator, weighing the new iterate against the previous state.}

\textcolor{black}{Implicit differentiation would instead evaluate $(I-J_T^*)^{-1}$ at a root, whereas we read $\vz^{(K)}$ at a nonzero residual (Table~\ref{tab:nofp}) and the undamped map is expansive on four of the five benchmarks (Table~\ref{tab:fixedpoint}), so the Neumann series for that inverse does not converge there.}
All results are means over three seeds, reported on the validation (dev) split of each benchmark; for VCR we evaluate on the first 5{,}000 validation examples.
CMU-MOSEI dataset is a tri-modal problem, for which we use a simple hierarchical approach: combine the image and audio modality first, then combine with text (Table~\ref{tab:encoders}).

\paragraph{Hyperparameters} All runs use AdamW, a linear warmup of 5\%, fixed-point tolerance $10^{-3}$, and $K{=}10$ damped iterations.
\begin{table}[h]
\centering
\caption{Training hyperparameters. \textcolor{black}{$\sigma(d)$ denotes a learned scalar passed through a sigmoid, initialised at $d=0.5$ so that $\beta$ begins at $0.62$; the other three benchmarks hold $\beta$ fixed.}}
\label{tab:hyper}
\begin{tabular}{lccccc}
\toprule
 & SNLI-VE & Hateful & VCR & MOSEI & VQA v2 \\
\midrule
learning rate & 1e-4 & 1e-4 & 1e-4 & 2e-5 & 1e-4 \\
epochs & 10 & 10 & 15 & 8 & 10 \\
batch size & 32 & 32 & 8 & 32 & 32 \\
weight decay & 0.01 & 0.01 & 0.01 & 1e-4 & 0.01 \\
$\lambda_j$ & 0.5 & 0.5 & 0.5 & 0.5 & 0.5 \\
$\lambda_f$ & 0.3 & 0.3 & 0.3 & 0.3 & 0 \\
Hutchinson probes & 5 & 1 & 1 & 5 & 1 \\
\textcolor{black}{damping $\beta$} & \textcolor{black}{$\sigma(d)$} & \textcolor{black}{0.5} & \textcolor{black}{0.5} & \textcolor{black}{$\sigma(d)$} & \textcolor{black}{0.5} \\
\bottomrule
\end{tabular}
\end{table}
\paragraph{Input encoders} Table~\ref{tab:encoders} is a summary of input encoders used in our experiments. All of our experiments use pretrained state-of-the-art encoders for both image and text modalities.
\begin{table}[h]
\centering
\caption{{Input encoders per benchmark (all frozen).}}
\label{tab:encoders}
\begin{tabular}{lll}
\toprule
Dataset & Modality-1 encoder & Modality-2 encoder \\
\midrule
SNLI-VE & CLIP image & BERT-base \\
Hateful Memes & CLIP image & BERT-base \\
VQA v2 & CLIP image & BERT-base \\
VCR & CLIP image + RoI features & BERT-base \\
CMU-MOSEI & \makecell[l]{COVAREP (audio \cite{covarep})\\FACET (visual~\cite{stoeckli2018facial})} & BERT-base\\

\bottomrule
\end{tabular}
\end{table}
\subsection{Extra Task Results}
\label{app:extra_task}
\begin{table}[hbt]%
\centering
\caption{Robustness to input corruption on Hateful Memes (dev AUROC).}%
\label{tab:robust}
\begin{tabular}{lcc}
\toprule
Perturbation & Concat. & \modelname \\
\midrule
\multicolumn{3}{c}{\textit{Image Perturbations}} \\
Noise ($\sigma{=}0.1$) & {0.6900} & {\textbf{0.7017}}\\
Noise ($\sigma{=}0.3$) & {0.6603} & {\textbf{0.6791}}\\
Blur ($r{=}2$) & {0.6850} & \textbf{0.6923}\\
Blur ($r{=}8$) & \textbf{0.6582} & 0.6564\\
\midrule
\multicolumn{3}{c}{\textit{Text Perturbations}} \\
Typo (5\%) & {0.6621} & \textbf{0.6791}\\
Typo (10\%) & {0.6508} & \textbf{0.6623}\\
\bottomrule
\end{tabular}
\end{table}%
\paragraph{Iterative refinement example}
We report an additional experiment on the SNLI-VE dataset to show that prediction accuracy increases over time. Figure~\ref{fig:stepacc} shows that accuracy rises steeply over the first three updates and is flat from $k{=}5$ onward, so the benefit comes from the early trajectory rather than from approaching a root.
\begin{figure}[t]
\centering
\includegraphics[width=0.55\linewidth]{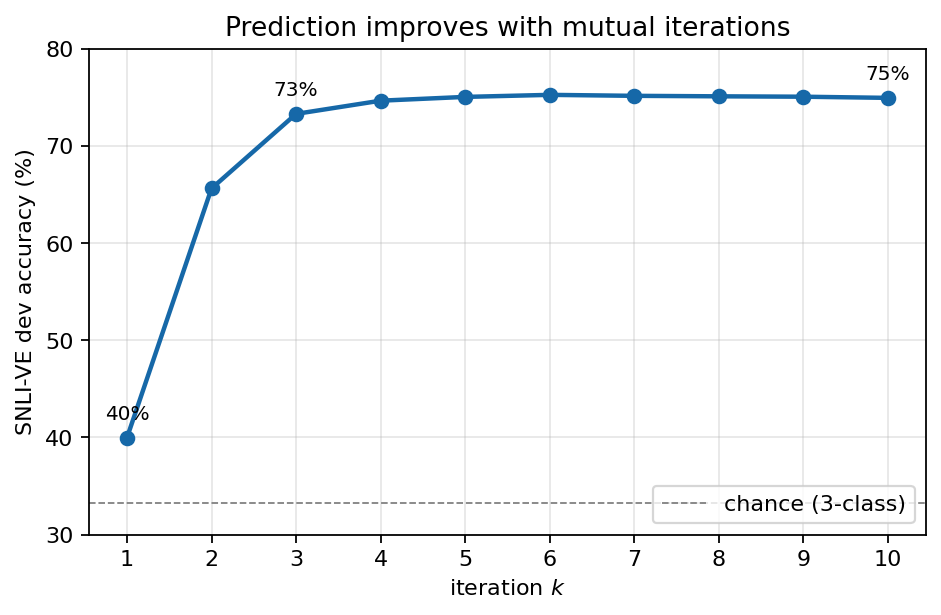}
\caption{{SNLI-VE dev accuracy read at each iteration $k$. }}
\label{fig:stepacc}
\end{figure}

\paragraph{Robustness Analysis}
We test robustness to input perturbation on the Hateful Memes dataset. The images were subjected to noise and blur perturbations while the texts were given typos. Table~\ref{tab:robust} shows the results over three seeds. The stochastic corruptions additionally average five independent draws, with all checkpoints seeing identical corrupted inputs within a draw. \\
Under strong additive image noise \modelname{} leads by $1.88$\,pp ($\sigma{=}0.3$, paired $2\sigma$ of $1.51$\,pp over three seeds and five corruption draws). 

\paragraph{COD10k results} Direct training of \modelname{} and the baseline architectures on the camouflage benchmark of Figure~\ref{fig:qual}. Results are averaged over three random seeds on a binary visual-entailment task constructed from COD10K images and automatically generated captions. Learning rates are chosen per architecture for training stability ($3{\times}10^{-5}$ for \modelname{} and concatenation, which otherwise collapses to a trivial optimum at $10^{-4}$; $10^{-4}$ for the attention baselines and low-rank multimodal fusion, LMF~\citep{lmf}). All other hyperparameters follow Table~\ref{tab:hyper}.
\begin{table}[h]
\centering
\caption{COD10k results.}\label{tab:cod}
\begin{tabular}{lccccc}
\toprule
 & Concat & Cross-attn & Self-attn & LMF & \modelname \\
\midrule
Accuracy & 0.7495 & 0.7924 & 0.7849 & 0.7799 & \textbf{0.7960} \\
$\Delta$ vs \modelname{} & $-$4.65pp & $-$0.35pp & $-$1.11pp & $-$1.61pp & -- \\
\bottomrule
\end{tabular}
\end{table}
\paragraph{Residual penalty effect} Final dev residual and metric, means over three seeds. On VQA v2 the small residual under $\lambda_f{=}0.3$ is the identity collapse rather than convergence: seed 42 ends at $7.62\times10^{-8}$ with the Jacobian norm at $1.000$.
\begin{table}[h]
\centering
\caption{Effect of the residual penalty.}
\label{tab:nofp}
\begin{tabular}{lcccc}
\toprule
& \multicolumn{2}{c}{Residual} & \multicolumn{2}{c}{Metric} \\
Dataset & $\lambda_{f}=0.3$ & $\lambda_{f}=0$ & $\lambda_{f}=0.3$ & $\lambda_{f}=0$ \\
\midrule
SNLI-VE & 1.0e-2 & 1.0e-1 & 75.15 & 75.06 \\
Hateful Memes & 4.2e-2 & 7.4e-2 & 70.21 & 70.37 \\
CMU-MOSEI & 2.3e-2 & 4.4e-2 & 54.20 & 54.97 \\
VCR & 1.1e-2 & 6.7e-2 & 62.95 & 60.49 \\
VQA v2 & 9.7e-4 & 9.4e-2 & 54.29 & 63.33 \\
\bottomrule
\end{tabular}
\end{table}
\subsection{Analysis Experiments}
\label{app:exp_analysis}
\paragraph{{When the residual penalty backfires}}
On VQA v2 the residual term admits a degenerate solution. The identity map sets $\Vert r(\vz;\vu)\Vert$ to exactly zero, and the optimization found it: the learned gate closed to approximately $10^{-7}$ on real inputs, so no cross-modal information flowed. Setting $\lambda_{f}=0$ on this benchmark alone restores the coupling and improves accuracy by $9.04$\,pp (Table~\ref{tab:nofp}). The other four benchmarks keep $\lambda_{f}=0.3$ (Table~\ref{tab:hyper}).

\paragraph{Can the collapse be prevented structurally?} Disabling the penalty on one benchmark is a configuration choice rather than a fix, so we also ran the structural alternative: bounding that gate away from zero and leaving the penalty on. With a floor of $0.1$ the constraint binds, since the gate settles at $0.1010$ across seeds, and accuracy recovers $2.30$\,pp of the $9.04$\,pp gained by removing the penalty (Table~\ref{tab:gatefloor}). The coupling does not come back with it: the cross-modal to self-modal sensitivity ratio on the image path is $0.0027$ under the floor against $0.0532$ with the penalty off, and the penalty-free model settles at $0.2339$, more than twice the floor it was never given. Bounding the gate closes one route to a zero residual while leaving another open, since the residual can also be driven down by shrinking the cross-modal function $M$ itself. A safeguard acting on the gate alone is therefore not sufficient. Neither is a bound on the norm of the whole Jacobian: the collapsed run sits at $\Vert\hat{J}_T\Vert=1.000$, comfortably above the lower end of the band \textcolor{black}{of $\mathcal{L}_{jac}$} in Equation~\ref{eq:loss}, while the off-diagonal blocks of Equation~\ref{eq:block} are exactly what goes to zero. Penalising those blocks directly is what the analysis in Section~\ref{sec:design} points to\textcolor{black}{, so we ran that as well: a hinge holding each off-diagonal block above the value the penalty-free model reaches, with the residual penalty left on. It works as a mechanism. All three seeds finish with more cross-modal sensitivity than the penalty-free model has, measured on held-out batches with dropout disabled, at $1.9$ to $15\times$ the floor. The accuracy does not follow: $53.35$, which against the collapsed arm is $-0.94$pp on a paired $2\sigma$ of $1.52$pp and is therefore not called. The two safeguards fail in opposite directions. Bounding the gate buys $2.30$pp without restoring the coupling, and restoring the coupling buys nothing, so what the residual penalty costs this benchmark is not the cross-modal path itself.}

\begin{table}[htb]
\centering
\caption{\textcolor{black}{Structural safeguards against the identity collapse, on VQA~v2. Means over three seeds. The last two columns are read at the epoch whose weights were kept, averaged over runs: the mean gate value on the image path, and that path's cross-modal sensitivity divided by its self-modal sensitivity.}}
\label{tab:gatefloor}
\begin{tabular}{lccccc}
\toprule
Configuration & Acc (\%) & $\Delta$ (pp) & called & gate & cross/self \\
\midrule
$\lambda_f = 0.3$, no floor & 54.29 & -- & & -- & -- \\
$\lambda_f = 0.3$, gate floor $0.1$ & 56.59 & $+$2.30 & yes & 0.1010 & 0.0027 \\
$\lambda_f = 0.3$, cross floor & 53.35 & $-$0.94 & no & 0.0894 & 0.0323 \\
$\lambda_f = 0$ & 63.33 & $+$9.04 & yes & 0.2339 & 0.0532 \\
\bottomrule
\end{tabular}
\end{table}

\paragraph{State design} \textcolor{black}{Each arm replaces only the state layout, with the iterated block widened until its parameter count matches. Matching is on the block: on VQA~v2 the single-state arm ends up $18.9\%$ larger in total, because widening the block also widens the projections feeding it and the head reading it, and it loses anyway. On Hateful Memes block-matching leaves that arm $12.7\%$ larger, so a second arm matched on the whole model is reported; the two agree. On CMU-MOSEI the single-state arm is the block \modelname{} replaced, and \modelname{}'s width was chosen by matching it, so all three blocks sit within $0.14\%$ of each other by construction. On Hateful Memes the baseline is a \modelname{} retrained alongside these arms rather than the run behind Table~\ref{tab:main}. The feature-sum arms are that fusion design under our encoders, splits and budget, not a reproduction of \citet{deqmm}, and their numbers are not its reported scores.}

 \textcolor{black}{On CMU-MOSEI the two manipulations separate: maintaining a state per modality is worth $0.84$\,pp, while letting those states read each other is not distinguishable from noise ($0.13$\,pp against a paired $2\sigma$ of $0.97$). Where the benchmark does reward fusion the cross-modal path is what matters instead, costing $4.20$\,pp on VQA~v2 (Table~\ref{tab:ablation}). The CMU-MOSEI single-state arm also reaches a final residual smaller by a factor of $25$ and still loses: converging better does not help here either, as in Figure~\ref{fig:stepacc}.}

\begin{table}[h]
\centering
\caption{ \textcolor{black}{State-design comparison. Each arm replaces only the state layout, at matched block parameter count, under the recipe of Table~\ref{tab:hyper}. Means over three seeds; a difference is called when it exceeds twice the sample standard deviation of the paired per-seed difference. Metrics differ per benchmark, so only the verdicts are comparable across blocks, not the sizes of the gaps.}}
\label{tab:statedesign}
\begin{tabular}{llccc}
\toprule
Dataset & Configuration & Score & $\Delta$ & called \\
\midrule
VQA v2 & \modelname{} (coupled states) & 63.33 & -- & \\
 & Single state, token level & 62.98 & $-$0.35 & yes \\
 & Feature-sum, one vector & 51.94 & $-$11.39 & yes \\
\midrule
Hateful Memes & \modelname{} (coupled states) & 0.7094 & -- & \\
 & Single state, block-matched & 0.7257 & $+$1.63 & no \\
 & Single state, model-matched & 0.7262 & $+$1.68 & no \\
 & Feature-sum, one vector & 0.6905 & $-$1.89 & no \\
\midrule
CMU-MOSEI & \modelname{} (coupled states) & 54.70 & -- & \\
 & Single state (replaced block) & 53.86 & $-$0.84 & yes \\
 & Feature-sum, one vector & 53.99 & $-$0.72 & no \\
\bottomrule
\end{tabular}
\end{table}

\begin{table}[thb]
\centering
\caption{ \textcolor{black}{Fixed-point behavior at the read depth $K{=}10$, measured on the full evaluation set of each benchmark for a single seed. \textcolor{black}{On CMU-MOSEI $n$ is $1{,}861$ of the $1{,}871$ validation clips: the text and the acoustic-visual features are distributed separately and ten clips fail to align.}} %
}
\label{tab:fixedpoint}

\begin{tabular}{lcccccccc}
\toprule
 & & & & \multicolumn{2}{c}{Residual} & \multicolumn{2}{c}{Metric} & \\
Dataset & Seed & $\rho$ at $\vz^{(K)}$ & Drift & at $\vz^{(K)}$ & at $\vz^{(300)}$ & at $\vz^{(K)}$ & at $\vz^{(300)}$ & $n$ \\
\midrule
CMU-MOSEI     & 42 & 0.965 & 0.38 & 3.2e-2 & 8.0e-6 & 53.9 & 51.8 & 1{,}861 \\
VCR           & 42 & 1.011 & 0.60 & 1.6e-2 & 1.0e-3 & 62.8 & 55.5 & 5{,}000 \\
Hateful Memes & 43 & 1.083 & 1.05 & 3.3e-2 & 1.5e-3 & 70.7 & 56.6 & 500 \\
VQA v2        & 42 & 1.340 & 1.42 & 7.4e-2 & 1.5e-3 & 62.8 & 27.4 & 21{,}029 \\
SNLI-VE       & 43 & 1.689 & 4.51 & 1.1e-2 & 4.4e-3 & 75.0 & 53.2 & 9{,}602 \\
\bottomrule
\end{tabular}

\end{table}

\paragraph{Fixed-point behavior} This experiment shows the per-dataset examinations of how the fixed-points behave (Table~\ref{tab:fixedpoint}). $\rho$ is the spectral radius of the undamped map and drift is $\Vert\vz^{(300)}-\vz^{(K)}\Vert/\Vert\vz^{(K)}\Vert$. Only CMU-MOSEI is contractive at the read depth, and on every benchmark iterating far past $K$ degrades the metric. This also empirically shows that it is beneficial to have a fixed point that can be reached in small number of iterations. \textcolor{black}{The iteration does reduce the residual, however: it falls by a factor of three to eighty before the read depth is reached, from $0.81$ to $1.1\times10^{-2}$ on SNLI-VE and from $0.84$ to $3.3\times10^{-2}$ on Hateful Memes, and the residual columns show it continuing to fall past it. The decrease is monotone in $k$ on four of the five; SNLI-VE is not monotone before $k{=}5$. The truncation therefore results in an approximate fixed point rather than an unrelated intermediate state.}

\paragraph{What the iteration costs} \textcolor{black}{We measure the cost of the iteration at inference: forward pass only, batch $32$, one RTX~3090, on the SNLI-VE architecture with random weights and random inputs (Table~\ref{tab:cost}). At the read depth $K{=}10$ \modelname{} costs $2.13\times$ the fastest single-pass baseline in wall clock and $1.87\times$ in FLOPs. The marginal cost is flat: each iteration adds $1.687$ GFLOPs per sample, constant to within $0.0003$ across $K \in \{1,2,5,10,20\}$, on an encoder base that the sweep extrapolates to $19.204$ GFLOPs against the $19.266$ the concatenation baseline measures directly. Peak memory does not move with $K$, since the forward carries one state rather than a stack of $K$ activations. Stopping early does not help here: left to run to a cap of $100$ with the tolerance in charge, the iteration reaches the cap, which is the finding of Table~\ref{tab:fixedpoint} seen from the cost side.}

\begin{table}[t]
\centering
\caption{\textcolor{black}{Inference cost at batch $32$ on one RTX~3090, SNLI-VE architecture. Wall clock is the median of $30$ timed passes after five warmups; FLOPs are counted for the whole forward including the frozen encoders; peak memory is measured per arm over an idle baseline. Random weights and random inputs, so this is the cost of a forward pass and not a statement about accuracy.}}
\label{tab:cost}
\begin{tabular}{lccc}
\toprule
Model & ms / sample & GFLOPs / sample & peak MB \\
\midrule
Concat MLP & 1.465 & 19.27 & 1015 \\
LMF & 1.483 & 19.27 & 1003 \\
Cross-attention & 1.621 & 21.20 & 1004 \\
Self-attention & 1.745 & 21.96 & 1006 \\
\midrule
\modelname{}, $K{=}1$ & 1.627 & 20.89 & 1003 \\
\modelname{}, $K{=}5$ & 2.304 & 27.64 & 1003 \\
\modelname{}, $K{=}10$ (read depth) & 3.127 & 36.08 & 1003 \\
\modelname{}, $K{=}20$ & 4.778 & 52.95 & 1003 \\
\bottomrule
\end{tabular}
\end{table}

\end{document}